\documentclass{article}
\usepackage[square,numbers]{natbib}

\usepackage[final]{neurips_mlsb_2024}

\usepackage[utf8]{inputenc} %
\usepackage[T1]{fontenc}    %
\usepackage{hyperref}       %
\usepackage{url}            %
\usepackage{booktabs}       %
\usepackage{amsfonts}       %
\usepackage{nicefrac}       %
\usepackage{microtype}      %
\usepackage{xcolor}         %
\usepackage{amsmath}
\usepackage{graphicx}
\usepackage{multirow} %
\usepackage{wrapfig,lipsum,booktabs}
\usepackage[english]{babel}
\usepackage{amsmath}
\usepackage{amsthm}
\usepackage{cleveref}
\usepackage{float} %
\newtheorem{theorem}{Theorem}

\title{\textsc{MolMix}: A Simple Yet Effective Baseline for Multimodal Molecular Representation Learning}

\author{Andrei Manolache\textsuperscript{123} \quad Dragoș Țânțaru\textsuperscript{3} \quad Mathias Niepert\textsuperscript{12} \\
\textsuperscript{1}Computer Science Department, University of Stuttgart, Germany \\
\textsuperscript{2}International Max Planck Research School for Intelligent Systems \quad \textsuperscript{3}Bitdefender, Romania \\
\texttt{andrei.manolache@ki.uni-stuttgart.de}
}

\begin{document}

\maketitle

\begin{abstract}
	In this work, we propose a simple transformer-based baseline for multimodal molecular representation learning, integrating three distinct modalities: SMILES strings, 2D graph representations, and 3D conformers of molecules. A key aspect of our approach is the aggregation of 3D conformers, allowing the model to account for the fact that molecules can adopt multiple conformations—an important factor for accurate molecular representation. The tokens for each modality are extracted using modality-specific encoders: a transformer for SMILES strings, a message-passing neural network for 2D graphs, and an equivariant neural network for 3D conformers. The flexibility and modularity of this framework enable easy adaptation and replacement of these encoders, making the model highly versatile for different molecular tasks. The extracted tokens are then combined into a unified multimodal sequence, which is processed by a downstream transformer for prediction tasks. To efficiently scale our model for large multimodal datasets, we utilize Flash Attention 2 and bfloat16 precision. Despite its simplicity, our approach achieves state-of-the-art results across multiple datasets, demonstrating its effectiveness as a strong baseline for multimodal molecular representation learning. Our code is publicly available at \url{https://github.com/andreimano/MolMix}.
\end{abstract}

\section{Introduction and Related Work}

Accurately representing molecular structures is fundamental in computational chemistry and drug discovery \citep{xia2023systematic, WIEDER20201, app_dl_mol}. Effective molecular representations enable machine learning models to predict molecular properties, understand chemical behaviors, and accelerate the development of new compounds. Traditional molecular representation methods typically focus on a single modality, such as SMILES strings \citep{Hirohara2018ConvolutionalNN, smilesbert}, chemical fingerprints \citep{fingerprints}, 2D molecular graphs \citep{Gil+2017}, or the 3D geometry of molecules \citep{Sch+2017, Kli+2021}. While effective, these methods overlook important molecular characteristics that can be captured by other modalities \citep{descriptors_enough, challenge_reaction}. To address this, recent research has introduced multimodal approaches that integrate multiple molecular representations to provide richer representations for machine learning tasks in chemistry. \citet{pmlr-v162-stark22a} proposed an information maximization approach to enhance the mutual information between 2D and 3D molecular embeddings. Similarly, \citet{liu2022pretraining} used contrastive pre-training to align 2D and 3D representations. Other approaches extract both 2D and 3D features, such as shortest path distances and 3D distance encodings, to build multimodal models \citep{luo2023one, yu2023multimodalmolecularpretrainingmodality}. \citet{10.1145/3534678.3539368} unified 2D and 3D molecular data in a pre-training framework by predicting either masked 2D structures or 3D conformations. Additionally, language-based models have been integrated with molecular data. \citet{10.1093/bioinformatics/btae260}, \citet{xiao2024proteingptmultimodalllmprotein}, and \citet{srinivas2024crossmodallearningchemistryproperty} leveraged large-scale language models to incorporate textual descriptions of molecules, enhancing molecular property predictions.

While these representations capture certain aspects of molecular structures, they may not fully encompass the variability inherent in molecular conformations. Many molecular properties, such as solubility, toxicity, and binding affinity, are influenced by the range of conformations a molecule can adopt in nature \citep{cao2022design}. Utilizing a single geometric representation for a molecule, therefore, restricts the effectiveness of machine learning models. Moreover, identifying which conformers most significantly impact the properties of interest remains difficult. Consequently, creating comprehensive multi-modal representations that integrate multiple 3D conformations continues to be a challenge. To address this, \citet{axelrod_molecular_2023} propose a scheme where 3D conformer embeddings are first extracted using an equivariant backbone and then aggregated with an attention mechanism, but they do not train on multimodal data. Similarly, \citet{nguyen2024structure} introduce a conformer aggregation approach that leverages Optimal Transport techniques to obtain a single 3D embedding from multiple conformers, which is subsequently combined with 2D embeddings derived from a GNN. Additionally, \citet{zhu2024learning} propose MARCEL, a conformer aggregation benchmark alongside models that employ set encoders to pool conformer embeddings with 2D structures and SMILES strings for downstream tasks. These approaches illustrate ongoing efforts to develop more holistic and effective molecular representations by incorporating both 2D and multiple 3D conformational data.

Despite these advances, there remains a need for simple yet effective models that can seamlessly integrate multiple modalities and handle multiple conformers without significant computational overhead. Moreover, recent observations \citep{swallowing-the-bitter-pill, wdtgg} suggest that some model design choices might be unnecessary for strong empirical performance, thereby making the added complexity superfluous and inefficient. To address this challenge, we propose \textsc{MolMix}, a simple yet effective baseline for multimodal molecular representation learning. We employ modality-specific encoders - a transformer for SMILES strings, a GNN for 2D graphs, and equivariant neural networks for 3D conformers - to extract text and node embeddings from each modality. These embeddings are concatenated into a multimodal sequence, separated by special tokens, and fed into a downstream transformer that predicts molecular properties. By leveraging efficient techniques like Flash Attention \citep{dao2022flashattention, dao2023flashattention2} and bfloat16 precision, \textsc{MolMix} scales to handle large sequences of atom tokens with minimal computational overhead, enabling the direct incorporation of all conformers. Despite its straightforward design, \textsc{MolMix} achieves state-of-the-art results across multiple datasets, demonstrating that simplicity can be highly effective in multimodal molecular representation learning, while the modular design allows us to easily exchange the specific modality encoders.

To summarize, our main contributions are:

\begin{enumerate}
	\item \textbf{Simple multimodal molecular framework}: We introduce \textsc{MolMix}, which seamlessly combines SMILES strings, 2D molecular graphs, and multiple 3D conformers into a unified sequence for molecular representation learning.
	\item \textbf{Conformer aggregation}: By incorporating node embeddings from 3D conformers, \textsc{MolMix} effectively captures conformational variability.
	\item \textbf{Scalability}: We utilize Flash Attention and bfloat16 (bf16) precision to scale our model, enabling the processing of large multimodal datasets with minimal computational overhead.
	\item \textbf{State-of-the-Art performance}: \textsc{MolMix} achieves superior results on multiple benchmark datasets, establishing a strong baseline for future research in multimodal molecular representation learning.
	\item \textbf{Transfer learning capabilities} We show that \textsc{MolMix} could potentially be used for pre-training on large molecular datasets.
\end{enumerate}

We make our code publicly available at \url{https://github.com/andreimano/MolMix}.

\section{\textsc{MolMix}: A Multimodal Molecular Transformer}

\paragraph{1D Encoder}

We represent molecules using SMILES strings, which encode chemical structures as sequences of characters. Let \( S = [s_1, s_2, \dots, s_n] \) denote the input SMILES string, where each \( s_i \) is a character. Each \( s_i \) is mapped to an embedding vector \( \mathbf{e}_i = \text{Embedding}(s_i) \). To account for sequence order, positional encodings are added: \( \mathbf{z}_i = \mathbf{e}_i + \text{PE}(i) \). A transformer encoder \citep{Vaswani2017-tj} then processes these vectors to obtain the hidden representations

\begin{equation}
	\label{eq: smiles}
	\mathbf{h}_i^{\text{1D}} = \text{Transformer}(\mathbf{z}_i),
\end{equation}

for all \( i \in \{1, \dots, n\} \). Each hidden representation \( \mathbf{h}_i^{\text{1D}} \) corresponds to the respective input character \( s_i \), effectively capturing the contextual information about the molecule, for each character.

\paragraph{2D Encoder}
We represent molecules as graphs \( G = (V, E) \), where \( V \) is the set of atoms and \( E \) is the set of covalent bonds. Each atom \( v \in V \) and bond \( e_{uv} \in E \) are associated with initial feature vectors \( \mathbf{x}_v \) and \( \mathbf{e}_{uv} \), respectively. We use a message-passing framework with GINE \citep{xu2018how, Hu2020Strategies} as the backbone to capture the molecular graph's structural information. At each message-passing step \(j\), the hidden representation of atom \( v \) is updated as

\begin{equation}
	\label{eq: mpnn}
	\mathbf{h}_{v, j}^{\text{2D}} = \text{GINE}\left(\mathbf{h}_{v, j-1}^{\text{2D}}, \{ \mathbf{h}_{u, j-1}^{\text{2D}} \mid u \in \mathcal{N}(v) \}, \{ \mathbf{e}_{uv} \}\right),
\end{equation}

where \( \mathcal{N}(v) \) denotes the neighbors of atom \( v \). This iterative process aggregates information from neighboring atoms and bonds, enabling the model to learn graph representations. The final hidden embeddings \( \mathbf{h}_{v, j}^{\text{2D}} \) encode both local and global structural features of the molecule.

\paragraph{3D Encoder}
To leverage the three-dimensional structural information of molecules, we utilize 3D conformations represented by the spatial coordinates of each atom. Let \( V \) denote the set of atoms. Each atom \( v \in V \) is associated with a 3D coordinate \( \mathbf{r}_v \in \mathbb{R}^3 \). To extract meaningful atom embeddings that respect the geometric properties of the molecule, we employ an neural network with 3D inductive biases, such as SchNet \citep{Sch+2017} or GemNet \citep{Kli+2021}, as the backbone model. These networks process the 3D coordinates \( \{\mathbf{r}_v\}_{v \in V} \) along with the initial atom features \( \{x_v\}_{v \in V}\) and apply a cutoff function to consider interactions within a specified distance range, generating hidden embeddings

\begin{equation}
	\label{eq: 3d}
	\mathbf{h}_v^{\text{3D}} = \text{3DNetwork}(\mathbf{r}_v, \mathbf{x}_v),
\end{equation}

for all \( v \in V \). These atom embeddings \( \mathbf{h}_v^{\text{3D}} \) capture both the local geometry and the global spatial arrangement of the molecule.

\paragraph{Downstream Multimodal Transformer}
To integrate different molecular representations, we design a multimodal transformer that combines three distinct modalities. The SMILES encoder outputs token embeddings \( \mathbf{h}_i^{\text{1D}} \), where \( \mathbf{h}_i^{\text{1D}} \) corresponds to the \( i^{\text{th}} \) character in the string. From the 2D MPNN encoder, we extract atom embeddings \( \mathbf{h}_{v,j}^{\text{2D}} \) for atom \( v \) at layer \( j \). By using embeddings from all layers, the model captures both local and distant atom interactions, mitigating the oversmoothing effect common in deep GNNs. The 3D encoder provides atom embeddings \( \mathbf{h}_{v,c}^{\text{3D}} \) for atom \( v \) and conformer \( c \), encapsulating spatial geometry. We use multiple conformers by simply adding all the atom embeddings to the multimodal sequence. Modality-specific learnable encodings are added to the embeddings from each modality. These modality-enhanced embeddings are concatenated into a unified sequence, with special tokens included: a classification token \( \mathbf{h}_{\text{CLS}} \) is added at the start, and separation tokens \( \mathbf{h}_{\text{SEP}} \) are placed between modalities. The resulting input sequence is structured as

\[
	\mathbf{H} = \left[ \mathbf{h}_{\text{CLS}}, \{ \mathbf{h}_i^{\text{1D}} \}_i, \mathbf{h}_{\text{SEP}}, \left\{ \mathbf{h}_{v,j}^{\text{2D}} \right\}_{v,j}, \mathbf{h}_{\text{SEP}}, \left\{ \mathbf{h}_{v,c}^{\text{3D}} \right\}_{v,c}, \mathbf{h}_{\text{SEP}} \right].
\]

This sequence is then processed by the downstream transformer, which utilizes the self-attention mechanism to integrate and contextualize information across all modalities. After the transformer layers, the embedding corresponding to the classification token $\mathbf{h}^{out}_{CLS}$ is extracted and sent to a readout MLP to perform downstream tasks such as property prediction and molecular classification:

\begin{align}
	\label{eq: mmtf}
	\mathbf{h}^{\text{out}}_{\text{CLS}} & = \text{MultimodalTransformer}(\mathbf{H})         \\
	\label{eq: mmtf2}
	\hat{y}                              & = \text{MLP}(\mathbf{h}^{\text{out}}_{\text{CLS}})
\end{align}

We reduce memory overhead in our multimodal transformer with bfloat16 precision and Flash Attention 2 \citep{dao2023flashattention2}. See \cref{sec: fa} for details and a memory comparison with classical attention.

Since we use the same positional encoding for each \( \mathbf{h}_{v,i}^{\text{2D}} \) and \( \mathbf{h}_{v,c}^{\text{3D}} \), we maintain the permutation equivariance property of the 2D and 3D encoders. Another desirable property is for the model to preserve any invariance of the 3D encoder. Indeed, \textsc{MolMix} preserves these useful inductive biases:

\begin{theorem}
	\label{thm: invariance}
	Let $S$ be the SMILES string, $G$ be the 2D graph, and $\{c_1, \dots, c_k\}$ be a set of $k$ 3D conformers for a molecule. Let $\hat{y} = f_\theta(S, G, \{c_1, \dots, c_k\})$ be the output prediction obtained as described in \cref{eq: smiles} - \eqref{eq: mmtf2}. Let our 3D encoder be invariant to the actions of some group $\mathcal{G}$. Then $f_\theta$ is also invariant to any $T_1, \dots, T_k \in \mathcal{G}$, i.e. $f_\theta(S, G, \{T_1c_1, \dots, T_kc_k\}) = f_\theta(S, G, \{c_1, \dots, c_k\})$.
\end{theorem}

\begin{table*}[t!]
	\vspace{-3em}

	\centering
	\caption{Comparison between \textsc{MolMix} and other baselines on the \textit{Drugs-75K} and \textit{Kraken} datasets from MARCEL \citep{zhu2024learning}. \textbf{1D}, \textbf{2D} and \textbf{3D} represents training on the SMILES strings and molecule fingerprints, 2D molecular representations and 3D conformers. \textbf{Multimodal} represents training on all three modalities. The metric used is the Mean Absolute Error (MAE, $\downarrow$). \textbf{Bold} indicates the best-performing model, while \underline{underline} denotes the second-best. \textsc{MolMix} obtains the best results on 5 out of 7 properties, with second-best results obtained on two properties (Drugs-75K/$\chi$ and Kraken/BurL).}
	\label{tab: marcel}
	\resizebox{0.8\linewidth}{!}{%
		\begin{tabular}{llccccccc}
			\toprule
			                                                                       & Model           & \multicolumn{3}{c}{Drugs-75K}  & \multicolumn{4}{c}{Kraken}                                                                                                                                                                          \\
			\cmidrule(lr){3-5} \cmidrule(lr){6-9}
			                                                                       &                 & IP $\downarrow$                & EA $\downarrow$                & $\chi$ $\downarrow$            & B5 $\downarrow$                & L $\downarrow$                 & BurB5 $\downarrow$             & BurL $\downarrow$              \\
			\midrule
			\multirow{3}{*}{\makebox[0pt][c]{\rotatebox{90}{\textbf{1D}}}}         & RF              & 0.498{\tiny±0.003}             & 0.474{\tiny±0.002}             & 0.273{\tiny±0.003}             & 0.476{\tiny±0.004}             & 0.430{\tiny±0.009}             & 0.275{\tiny±0.018}             & 0.152{\tiny±0.014}             \\
			                                                                       & LSTM            & 0.478{\tiny±0.002}             & 0.464{\tiny±0.000}             & 0.250{\tiny±0.005}             & 0.487{\tiny±0.028}             & 0.514{\tiny±0.041}             & 0.281{\tiny±0.004}             & 0.192{\tiny±0.002}             \\
			                                                                       & Transformer     & 0.661{\tiny±0.002}             & 0.585{\tiny±0.003}             & 0.407{\tiny±0.001}             & 0.961{\tiny±0.081}             & 0.839{\tiny±0.043}             & 0.493{\tiny±0.037}             & 0.278{\tiny±0.021}             \\
			\midrule
			\multirow{4}{*}{\makebox[0pt][c]{\rotatebox{90}{\textbf{2D}}}}         & GIN             & 0.435{\tiny±0.003}             & 0.417{\tiny±0.003}             & 0.226{\tiny±0.002}             & 0.313{\tiny±0.026}             & 0.400{\tiny±0.034}             & 0.172{\tiny±0.003}             & 0.120{\tiny±0.004}             \\
			                                                                       & GIN+VN          & 0.436{\tiny±0.006}             & 0.417{\tiny±0.008}             & 0.227{\tiny±0.000}             & 0.357{\tiny±0.003}             & 0.434{\tiny±0.042}             & 0.242{\tiny±0.003}             & 0.174{\tiny±0.011}             \\
			                                                                       & ChemProp        & 0.460{\tiny±0.003}             & 0.442{\tiny±0.005}             & 0.244{\tiny±0.001}             & 0.485{\tiny±0.007}             & 0.545{\tiny±0.045}             & 0.300{\tiny±0.009}             & 0.195{\tiny±0.014}             \\
			                                                                       & GraphGPS        & 0.435{\tiny±0.005}             & 0.409{\tiny±0.006}             & 0.221{\tiny±0.005}             & 0.345{\tiny±0.032}             & 0.436{\tiny±0.013}             & 0.207{\tiny±0.012}             & 0.150{\tiny±0.014}             \\
			\midrule
			\multirow{6}{*}{\makebox[0pt][c]{\rotatebox{90}{\textbf{3D}}}}         & SchNet          & 0.439{\tiny±0.006}             & 0.421{\tiny±0.002}             & 0.224{\tiny±0.009}             & 0.329{\tiny±0.007}             & 0.546{\tiny±0.034}             & 0.230{\tiny±0.011}             & 0.186{\tiny±0.010}             \\
			                                                                       & DimeNet++       & 0.444{\tiny±0.009}             & 0.423{\tiny±0.007}             & 0.244{\tiny±0.008}             & 0.351{\tiny±0.011}             & 0.417{\tiny±0.040}             & 0.210{\tiny±0.016}             & 0.153{\tiny±0.007}             \\
			                                                                       & GemNet          & \underline{0.407{\tiny±0.001}} & \underline{0.392{\tiny±0.002}} & \textbf{0.197{\tiny±0.004}}    & 0.279{\tiny±0.013}             & 0.375{\tiny±0.009}             & 0.178{\tiny±0.010}             & 0.164{\tiny±0.006}             \\
			                                                                       & PaiNN           & 0.451{\tiny±0.004}             & 0.450{\tiny±0.005}             & 0.232{\tiny±0.004}             & 0.344{\tiny±0.039}             & 0.447{\tiny±0.032}             & 0.240{\tiny±0.018}             & 0.167{\tiny±0.009}             \\
			                                                                       & ClofNet         & 0.439{\tiny±0.008}             & 0.425{\tiny±0.007}             & 0.238{\tiny±0.002}             & 0.487{\tiny±0.009}             & 0.642{\tiny±0.036}             & 0.288{\tiny±0.017}             & 0.253{\tiny±0.005}             \\
			                                                                       & LEFTNet         & 0.417{\tiny±0.001}             & 0.396{\tiny±0.001}             & 0.208{\tiny±0.005}             & 0.307{\tiny±0.001}             & 0.449{\tiny±0.026}             & 0.218{\tiny±0.001}             & 0.149{\tiny±0.010}             \\
			\midrule
			\multirow{7}{*}{\makebox[0pt][c]{\rotatebox{90}{\textbf{Multimodal}}}} & SchNet          & 0.454{\tiny±0.007}             & 0.438{\tiny±0.013}             & 0.237{\tiny±0.010}             & 0.270{\tiny±0.019}             & 0.432{\tiny±0.046}             & 0.202{\tiny±0.018}             & 0.144{\tiny±0.004}             \\
			                                                                       & DimeNet++       & 0.413{\tiny±0.008}             & 0.394{\tiny±0.003}             & 0.227{\tiny±0.005}             & 0.263{\tiny±0.012}             & 0.347{\tiny±0.009}             & 0.178{\tiny±0.011}             & 0.119{\tiny±0.011}             \\
			                                                                       & GemNet          & 0.419{\tiny±0.002}             & 0.400{\tiny±0.001}             & 0.217{\tiny±0.004}             & 0.231{\tiny±0.003}             & \underline{0.339{\tiny±0.027}} & 0.159{\tiny±0.007}             & \textbf{0.095{\tiny±0.001}}    \\
			                                                                       & PaiNN           & 0.447{\tiny±0.007}             & 0.427{\tiny±0.003}             & 0.229{\tiny±0.007}             & \underline{0.223{\tiny±0.022}} & 0.362{\tiny±0.019}             & \underline{0.169{\tiny±0.011}} & 0.132{\tiny±0.009}             \\
			                                                                       & ClofNet         & 0.428{\tiny±0.006}             & 0.403{\tiny±0.002}             & 0.220{\tiny±0.007}             & 0.323{\tiny±0.002}             & 0.449{\tiny±0.005}             & 0.218{\tiny±0.019}             & 0.155{\tiny±0.004}             \\
			                                                                       & LEFTNet         & 0.417{\tiny±0.004}             & 0.395{\tiny±0.000}             & \underline{0.207{\tiny±0.002}} & 0.264{\tiny±0.013}             & 0.364{\tiny±0.035}             & 0.202{\tiny±0.003}             & 0.139{\tiny±0.001}             \\
			                                                                       & \textsc{MolMix} & \textbf{0.405{\tiny±0.002}}    & \textbf{0.379{\tiny±0.004}}    & \underline{0.206{\tiny±0.002}} & \textbf{0.191{\tiny±0.017}}    & \textbf{0.305{\tiny±0.020}}    & \textbf{0.146{\tiny±0.002}}    & \underline{0.121{\tiny±0.005}} \\
			\bottomrule
		\end{tabular}
	}
\end{table*}

\section{Experimental Setup and Results}

\begin{wraptable}{r}{8cm}
	\vspace{-1.8em}
	\caption{Comparison between \textsc{MolMix} and other approaches as reported in \citep{nguyen2024structure}. The metric used is Root Mean Squared Error (RMSE $\downarrow$). \textsc{MolMix} obtains the overall best scores, significantly improving upon the recently proposed multimodal CONAN-FGW model.}
	\label{tab:moleculenet-eval}
	\scriptsize
	\begin{tabular}{lcccc}
		\toprule
		Model           & Lipo $\downarrow$              & ESOL $\downarrow$              & FreeSolv $\downarrow$          & BACE $\downarrow$              \\ \midrule
		2D-GAT          & 1.178{\tiny±0.454}             & 1.513{\tiny±0.130}             & 2.926{\tiny±1.160}             & 1.358{\tiny±0.574}             \\
		D-MPNN          & 0.731{\tiny±0.148}             & 0.961{\tiny±0.212}             & 2.053{\tiny±0.261}             & 0.850{\tiny±0.145}             \\
		MolFormer       & 0.701{\tiny±0.110}             & 0.875{\tiny±0.249}             & 2.342{\tiny±0.212}             & 1.045{\tiny±0.145}             \\
		SchNet-scalar   & 0.839{\tiny±0.179}             & 0.820{\tiny±0.164}             & 1.268{\tiny±0.397}             & 0.850{\tiny±0.316}             \\
		SchNet-emb      & 0.767{\tiny±0.179}             & 0.797{\tiny±0.239}             & 1.260{\tiny±0.369}             & 0.832{\tiny±0.167}             \\
		ChemProp3D      & 0.695{\tiny±0.230}             & 0.825{\tiny±0.152}             & 1.419{\tiny±0.427}             & 0.903{\tiny±0.412}             \\
		CONAN           & 0.746{\tiny±0.114}             & 0.756{\tiny±0.138}             & 1.223{\tiny±0.397}             & 0.797{\tiny±0.226}             \\
		CONAN-FGW       & \underline{0.650{\tiny±0.126}} & \underline{0.727{\tiny±0.148}} & \underline{1.033{\tiny±0.288}} & \underline{0.741{\tiny±0.126}} \\

		\textsc{MolMix} & \textbf{0.614{\tiny±0.022}}    & \textbf{0.639{\tiny±0.017}}    & \textbf{0.976{\tiny±0.044}}    & \textbf{0.387{\tiny±0.041}}    \\
		\bottomrule
	\end{tabular}
	\vspace{-3em}
\end{wraptable}

In this section, we evaluate how \textsc{MolMix} improves predictive performance on real-world datasets by addressing the following questions: \textit{\textbf{Q1}) How does \textsc{MolMix}'s performance compare to other sophisticated models?; \textbf{Q2}) Does incorporating multiple modalities enhance downstream performance?; \textbf{Q3}) Are pre-trained weights beneficial for transfer learning?}

To address Question 1, we train \textsc{MolMix} on four MoleculeNet datasets \citep{Wu+2018}—\textit{Lipo}, \textit{ESOL}, \textit{FreeSolv}, and \textit{BACE}—covering various molecular properties, including physical chemistry and biophysics. Conformers are generated using the RDKit chemoinformatics package \citep{landrum2016rdkit}. We follow the same train/validation/test splits as \citep{nguyen2024structure}. We also train models on the newly introduced MARCEL benchmark \citep{zhu2024learning}, specifically the \textit{Drugs-75k} and \textit{Kraken} datasets. \textit{Drugs-75K}, a subset of the \textit{GEOM-Drugs} dataset \citep{axelrod2022geom}, contains 75,099 molecules with conformers generated by Auto3D \citep{auto3d}, and labels for ionization potentials (IP), electron affinity (EA), and electronegativity (\( \chi \)). \textit{Kraken} \citep{kraken} includes 1,552 monodentate organophosphorus ligands, with conformers generated via DFT, and labels for four 3D ligand descriptors: Sterimol B5 (B5), Sterimol L (L), buried Sterimol B5 (BurB5), and buried Sterimol L (BurL). We follow the same splits as in \citep{zhu2024learning}.

\begin{wrapfigure}{R}{8cm}
    \vspace{-2em} 
    \begin{minipage}[t]{8cm}
        \centering
        \caption{Modality ablation study on the Kraken dataset (MAE $\downarrow$). We keep the downstream Transformer fixed and train using a single modality or a combination of modalities. Using all three modalities obtains the best results on three out of the four properties, with the second-best results generally being obtained by a configuration that contains 3D conformers. Notably, for the \textit{buried Sterimol L} property, the best results are obtained by a 3D encoder + Transformer model, indicating that the property could mainly depend on the 3D structure.}
        \label{tab:modality_ablation}
        \scriptsize
        \begin{tabular}{lcccc}
            \toprule
		Modality & B5 $\downarrow$                & L $\downarrow$                 & BurB5 $\downarrow$             & BurL $\downarrow$              \\ \midrule
		1D       & 0.499{\tiny±0.033}             & 0.497{\tiny±0.025}             & 0.291{\tiny±0.020}             & 0.187{\tiny±0.005}             \\
		2D       & 0.258{\tiny±0.017}             & 0.347{\tiny±0.013}             & 0.176{\tiny±0.010}             & 0.141{\tiny±0.004}             \\
		3D       & 0.213{\tiny±0.008}             & 0.337{\tiny±0.009}             & 0.164{\tiny±0.004}             & \textbf{0.116{\tiny±0.003}}    \\
		1D+2D    & 0.297{\tiny±0.015}             & 0.390{\tiny±0.016}             & 0.180{\tiny±0.008}             & 0.153{\tiny±0.006}             \\
		1D+3D    & 0.209{\tiny±0.002}             & \underline{0.337{\tiny±0.010}} & 0.156{\tiny±0.013}             & 0.127{\tiny±0.004}             \\
		2D+3D    & \underline{0.202{\tiny±0.009}} & 0.356{\tiny±0.026}             & \underline{0.151{\tiny±0.004}} & \underline{0.122{\tiny±0.004}} \\
		1D+2D+3D & \textbf{0.191{\tiny±0.017}}    & \textbf{0.305{\tiny±0.020}}    & \textbf{0.146{\tiny±0.002}}    & \underline{0.121{\tiny±0.005}} \\
            \bottomrule
        \end{tabular}
        \vspace{1em} 
        \caption{Transfer learning experiment. We select the best checkpoint of a model trained to predict the electronegativity (\(\chi\)) on the \textit{Drugs-75K} dataset. We then freeze the model and only train the last linear readout layer on the \textit{Kraken} dataset. We compare with a randomly initialized model. For all descriptors, using the pre-trained weights improve predictive performance. Note that pretraining improves both mean performance and standard deviations.}
        \label{tab:pretrain}
        \scriptsize
        \begin{tabular}{lcccc}
            \toprule
		Modality     & B5 $\downarrow$             & L $\downarrow$              & BurB5 $\downarrow$          & BurL $\downarrow$           \\ \midrule
		Random init. & 0.567{\tiny±0.010}          & 0.543{\tiny±0.020}          & 0.334{\tiny±0.004}          & 0.216{\tiny±0.003}          \\
		Pretrain     & \textbf{0.521{\tiny±0.003}} & \textbf{0.509{\tiny±0.004}} & \textbf{0.316{\tiny±0.001}} & \textbf{0.195{\tiny±0.003}} \\
            \bottomrule
        \end{tabular}
        \vspace{-3em}
    \end{minipage}
\end{wrapfigure}

\vspace{-1em}
For all experiments, we report the mean and standard deviation over five different runs. We use the same hyperparameters for the modality encoders—we encode the SMILES strings using a transformer with two layers and four attention heads, the 2D graph is processed by a GINE \citep{xu2018how} network containing 6 layers, and the 3D conformations are processed either using a SchNet \citep{Sch+2017} or a GemNet \citep{Kli+2021} model. All of the encoders have a hidden dimension of 128. The modality embeddings are then projected by a linear layer to a 512-dimensional space before we jointly processing them with a downstream Transformer network with 8 heads and 6 layers for the \textit{Lipo}, \textit{ESOL}, \textit{FreeSolv}, \textit{BACE} and \textit{Kraken}, and 12 heads with 8 layers on \textit{Drugs-75K}. We use the Schedule-Free AdamW optimizer \citep{Kin+2015, adamw, schedulefree}.

On the MoleculeNet datasets, \textsc{MolMix} obtains the overall best results, significantly improving upon the results of \citep{nguyen2024structure} on some datasets such as \textit{BACE}, as can be seen in \cref{tab:moleculenet-eval}. This highlights that our simple approach can potentially learn better multimodal representations with conformer aggregation than previously proposed methods that contain more sophisticated aggregation techniques.

On the MARCEL datasets, compared to the approaches in \citep{zhu2024learning}, we achieve the best results on five out of seven properties and second-best on the remaining two, as can be seen in \cref{tab: marcel}. This suggests that \textsc{MolMix} also consistently performs well when using physically-grounded conformer generation methods like DFT and Auto3D.

To address Question 2, \cref{tab:modality_ablation} shows that for three of the four \textit{Kraken} descriptors, training on all three modalities yields the best results, while for one property, training on the 3D modality alone performs slightly better. The results indicate that 3D tokens contribute the most to downstream performance, followed by 2D, with SMILES (1D) having the least impact. Notably, SMILES strings significantly improve performance when predicting the Sterimol L descriptor.

To answer Question 3, we select the best checkpoint from a model trained to predict electronegativity (\(\chi\)) on the \textit{Drugs-75K} dataset. We then freeze its weights and train only the final linear layer to predict descriptors on the \textit{Kraken} dataset, comparing it to a randomly-initialized model. As shown in \cref{tab:pretrain}, pre-training improves predictive performance in all cases, suggesting that with sufficient data, \textsc{MolMix} could serve as a foundation model for molecular tasks.

\section{Conclusions and Further Work}

We propose \textsc{MolMix}, a simple yet effective multimodal molecular transformer supporting conformer aggregation. \textsc{MolMix} preserves inductive biases of modality encoders and achieves state-of-the-art results across multiple datasets. We hint towards \textsc{MolMix} being able to support transfer learning, suggesting that it could be used as a molecular foundation model. Finally, we use Flash Attention and bf16 precision to handle longer sequences and multiple modalities efficiently.

We leave three open questions. First, large self-supervised VLMs excel in 0-shot prediction and fine-tuning \citep{vlm3, vlm4, vlm1, vlm2}. Exploring self-supervised pre-training for \textsc{MolMix} using signals like masked language modeling \citep{devlin-etal-2019-bert} and noise-contrastive estimation \citep{clark-etal-2020-pre} could be valuable. Second, multiple conformers without pooling may be suboptimal; token merging \citep{bolya2023token} could improve memory and runtime. Lastly, adding modalities like molecular fingerprints may enhance performance.

\subsubsection*{Acknowledgments}
AM and MN acknowledge funding by Deutsche Forschungsgemeinschaft (DFG, German Research Foundation) under Germany's Excellence Strategy - EXC 2075 – 390740016, the support by the Stuttgart Center for Simulation Science (SimTech), and the International Max Planck Research School for Intelligent Systems (IMPRS-IS). AM and DȚ acknowledge funding by the EU Horizon project ELIAS (No. 101120237).

\bibliographystyle{unsrtnat}
\bibliography{references}

\clearpage

\appendix

\makeatletter
\@addtoreset{theorem}{section}
\makeatother

\section{Proof for Theorem \ref{thm: invariance}}
\label{proof: invariance}

\begin{theorem}
	Let $S$ be the SMILES string, $G$ be the 2D graph, and $\{c_1, \dots, c_k\}$ be a set of $k$ 3D conformers for a molecule. Let $\hat{y} = f_\theta(S, G, \{c_1, \dots, c_k\})$ be the output prediction obtained as described in \cref{eq: smiles} - \eqref{eq: mmtf2}. Let our 3D encoder be invariant to the actions of some group $\mathcal{G}$. Then $f_\theta$ is also invariant to any $T_1, \dots, T_k \in \mathcal{G}$, i.e. $f_\theta(S, G, \{T_1c_1, \dots, T_kc_k\}) = f_\theta(S, G, \{c_1, \dots, c_k\})$.
\end{theorem}
\begin{proof}
	Let $g_\theta$ be our 3D encoder network, as described in \cref{eq: 3d}. Let $V$ be the set of atoms and $\{x_v\}_{v\in V}$, $\{r_v\}_{v\in V}$ the atom features and their 3D coordinates, such that a conformer can be described as the tuple $c=(\{x_v\}_{v\in V}, \{r_v\}_{v\in V})$. We assume that $g_\theta$ is invariant to any action $T \in \mathcal{G}$, therefore we have that, for any conformer $c$,  $g_\theta(Tc)=g_\theta(c)=\mathbf{h}^{\text{3D}}$.

	Let $h_\theta$ be the downstream Transformer together with the readout layer, as described in \cref{eq: mmtf} - \eqref{eq: mmtf2}. Since we add the same learnable modality encoding to each $h^{\text{3D}}_{v,k}$, we also have that for any permutation $\pi \in Sym(K)$, we have

	\begin{align*}
		h_\theta(\{g_\theta(T_1c_1), \dots, g_\theta(T_kc_k)\}) & = h_\theta(\{g_\theta(c_1), \dots, g_\theta(c_k)\})                                                    \\
		                                                        & = h_\theta(\{ \mathbf{h}_{v,1}^{\text{3D}}, \dots, \mathbf{h}_{v,k}^{\text{3D}} \}_{v\in V})           \\
		                                                        & = h_\theta(\{ \mathbf{h}_{v,\pi(1)}^{\text{3D}}, \dots, \mathbf{h}_{v,\pi(k)}^{\text{3D}} \}_{v\in V}) \\
		                                                        & =\tilde{y},
	\end{align*}

	therefore, if when we include the 2D graph $G$ and the SMILES string $S$, we obtain $f_\theta(S, G, \{T_1c_1, \dots, T_kc_k\}) = f_\theta(S, G, \{c_1, \dots, c_k\})=\hat{y}$.
\end{proof}

\section{Qualitative attention example}
\label{sec: qual}

\begin{figure}[t!]
	\centering
	\includegraphics[width=\textwidth]{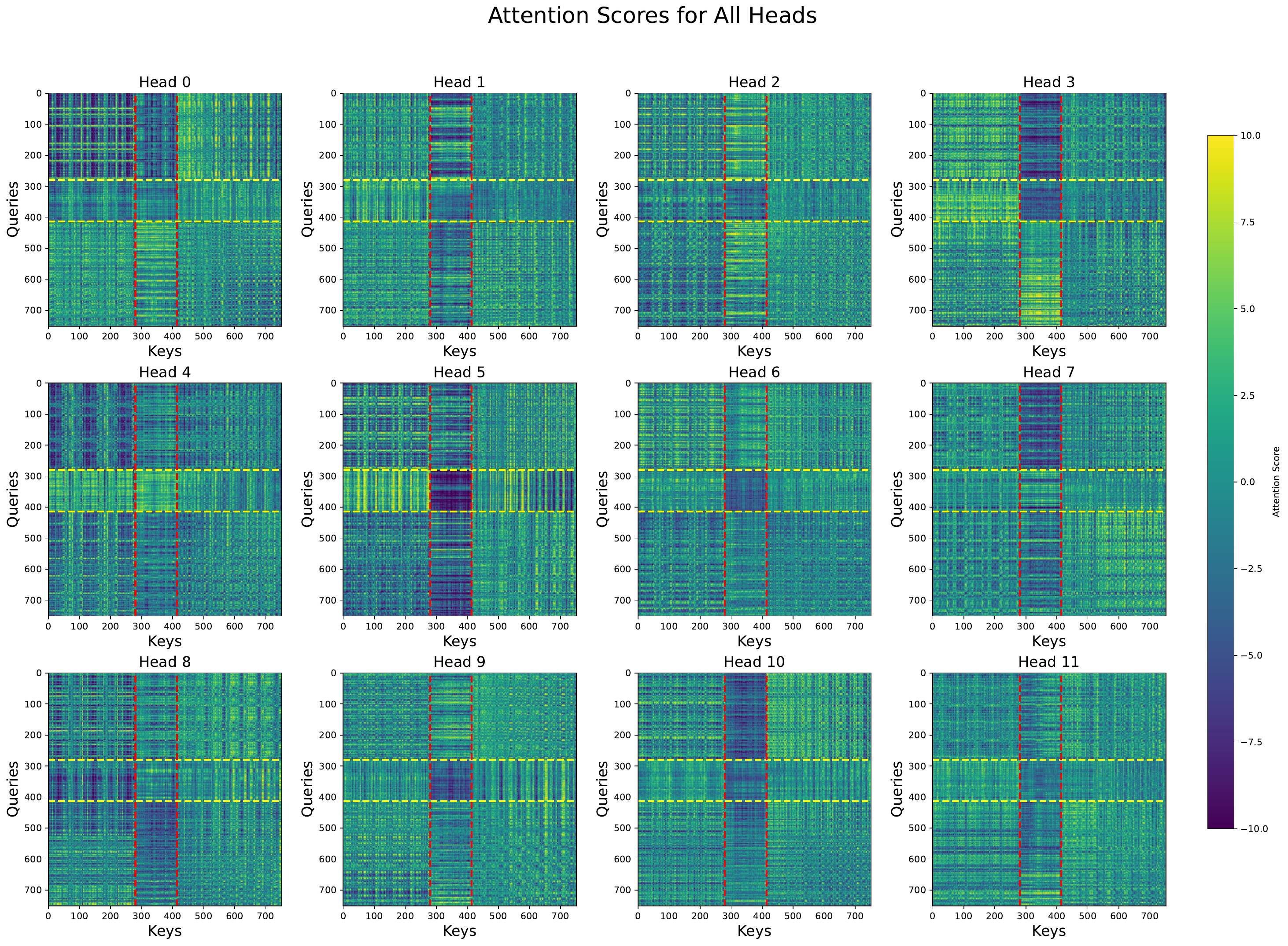}
	\caption{Red lines mark the boundaries between modalities on the key axis (with keys for each token represented by columns), while yellow lines mark the boundaries on the query axis (with queries represented by rows). The modalities are ordered as 3D, SMILES, and 2D. The attention scores are taken from the first layer of the model and clipped to the \([-10, 10]\) range.}
	\label{fig:heads_softcap}
\end{figure}

We present the attention scores for each head in a \textsc{MolMix} model trained on the Drugs-75k dataset, using a randomly sampled molecule from the dataset for visualization. As shown in Figure \ref{fig:heads_softcap}, distinct patterns emerge across the attention heads. While it remains challenging to assign a definitive interpretation to each individual head, certain sparse or dense patterns are evident in each cross-modality section. This suggests that the model is learning to extract meaningful and potentially useful features from all modalities.

\section{Attention implementation details}
\label{sec: fa}

We employ Flash Attention 2 \citep{dao2023flashattention2} for the self-attention mechanism in our models. Flash Attention 2 is a hardware-optimized implementation that significantly reduces both memory usage and runtime compared to the standard attention algorithm. It achieves these gains by leveraging GPU programming techniques, such as kernel fusion and tiling. Additionally, we utilize the \emph{varlen} implementation, which prevents unnecessary memory and compute consumption on padding tokens.

Table \ref{tab:attn_mem} presents the memory savings achieved by using Flash Attention 2 during training.

\begin{table*}[b!]
	\centering
	\caption{Comparison between fp32 standard attention and bf16 Flash Attention 2 memory usage across different models and batch sizes.}
	\label{tab:attn_mem}
	\resizebox{0.5\linewidth}{!}{%
		\begin{tabular}{lcccc}
			\toprule
			Dataset                     & \multicolumn{3}{c}{Kraken}                       \\
			\cmidrule(lr){2-4}
			Batch Size                  & 16                         & 32       & 64       \\
			\midrule
			Standard Attention (SchNet) & ~67 GB                     & OOM      & OOM      \\
			Standard Attention (GemNet) & ~80 GB                     & OOM      & OOM      \\
			Flash Attention 2 (SchNet)  & ~5.1 GB                    & ~11.6 GB & ~22.3 GB \\
			Flash Attention 2 (GemNet)  & ~22 GB                     & ~40 GB   & ~73 GB   \\
			\bottomrule
		\end{tabular}
	}
\end{table*}

\end{document}